\documentclass[runningheads]{llncs}
\usepackage[T1]{fontenc}
\usepackage{graphicx}
\usepackage{amsmath,amssymb,amsfonts}

\usepackage{multirow} 
\usepackage{csquotes} 
\usepackage{mathtools} 
\usepackage{subcaption} 
\usepackage{booktabs} 
\usepackage{longtable} 
\usepackage{latinabbrv} 
\usepackage{hyperref}
\usepackage{color}

\usepackage{xcolor}
\definecolor{tabblue}{RGB}{31,119,180}

\usepackage[capitalize]{cleveref}
\crefname{section}{Sec.}{Secs.}
\Crefname{section}{Section}{Sections}
\Crefname{table}{Table}{Tables}
\crefname{table}{Tab.}{Tabs.}
\crefname{lem}{Lem.}{Lems.}
\Crefname{lem}{Lemma}{Lemmas}
\crefname{ass}{Assum.}{Assums.}
\Crefname{ass}{Assumption}{Assumptions}

\newcommand{\Bcal}{\mathcal{B}}
\newcommand{\Fcal}{\mathcal{F}}
\newcommand{\Tcal}{\mathcal{T}}
\newcommand{\Symbs}{\Sigma}
\newcommand{\Sdef}{\Symbs_\mathrm{def}}
\newcommand{\dhat}{\widehat{d}}
\newcommand{\nhat}{\widehat{n}}
\newcommand{\R}{\mathbb{R}}
\newcommand{\Z}{\mathbb{Z}}
\newcommand{\indic}{\mathbf{1}}
\newcommand{\indep}{\perp}
\DeclareMathOperator{\cosdist}{CosineDist}
\DeclareMathOperator{\ICA}{ICA}
\DeclareMathOperator{\Prob}{\mathbb{P}}
\DeclareMathOperator{\Bin}{Bin}
\newcommand{\rmm}{\mathrm{rm}}
\newcommand{\itt}{\mathrm{it}}

\newcommand{\abs}[1]{\lvert#1\rvert}

\newtheorem{lem}{Lemma}
\newtheorem{ass}{Assumption}

\begin{document}
\title{Theatre Chapbooks At Scale: A Statistical Comparative Analysis of Typography\thanks{D.\ Belzarena and S.\ Mowlavi contributed equally to this work.}
}
\titlerunning{A Statistical Comparative Analysis of Typography}
%

\author{Diego Belzarena\inst{1,2}\orcidID{0009-0006-0535-7246} \and
Seginus Mowlavi\inst{1}\orcidID{0009-0005-7774-9217} \and
Paula Casariego Casti\~neira\inst{3}\orcidID{0000-0003-0527-2454} \and
Alejandra Ulla Lorenzo\inst{4}\orcidID{0000-0002-8137-9969} \and
Gregory Randall\inst{2}\orcidID{0000-0001-7911-2977} \and
Jean-Michel Morel\inst{5}\orcidID{0000-0002-6108-897X}
}
 \authorrunning{D. Belzarena, S. Mowlavi et al.}
 \institute{
 Universit\'e Paris-Saclay, ENS Paris-Saclay, Centre Borelli, France \and
 IIE, Facultad de Ingenería, Universidad de la Rep\'ublica, Uruguay \and
 Universit\`a Roma Tre, Italia \and
 Universidade de Santiago de Compostela, Espa\~na \and
 Lingnan University, Hong Kong
 }
\maketitle              
\begin{abstract}
    We propose a statistical methodology that quantifies the similarity of typefaces between printed historical books. This provides a tool that accelerates philological analysis. Using character prototypes derived from clustering and aligning automatically extracted character images, the method defines a typeface distance between any two books. To produce actionable outputs, we develop an \emph{a contrario} statistical framework to interpret the significance of the computed typeface distances. We apply the method to the philological study of $17^\mathrm{th}$-century Spanish printed theatre chapbooks in a quantity that exceeds the capabilities of systematic visual inspection by human experts. Our method enables the automatic comparison of Roman and Italic types extracted from different books. After validation by human experts, our method has led to new printer attributions being discovered, and former printer attributions being revised. This success strongly suggests that our method has the potential to enable digital bibliography on a larger scale than was previously possible.
\end{abstract}

\section{Introduction}

Many printers in the Iberian Peninsula in the 17th century tried to evade legal controls on book publishing by publishing without, or with false, editorial data, as evidenced by the fact that of the 78,524 volumes published in this territory during that century, 40,952 do not bear the name of the printer. With regard to printed theatre books, this situation was particularly pronounced in one of the most widespread publishing genres on the Peninsula during this period: the \emph{sueltas} or \emph{impresos sueltos} (theatre chapbooks), printed in quarto format, popular in nature, short in length, and of low material quality. The lack of correct bibliographical identification of this wide range of testimony is one of the great challenges in the study of the Spanish Golden Age theatrical heritage, since it has serious consequences for our knowledge of the seventeenth-century book trade and the role of this type of printed matter in the circulation of theatrical work~\cite{cruickshank1985,cruickshank1989,cruickshank1991,ulla2020a}.

The methodology used to identify these editions continues, to date, to follow a manual procedure based on the analysis of the typography and, if available, the ornamentation of each theatre chapbook, the results of which must be compared with the typography, graphic architecture, and ornamentation of printed books with unquestionable bibliographic data. Individual editions must be related to other printed materials with which they share significant common characteristics to establish a network of typographical relationships~\cite{wilson&cruickshank1980,casariego2023}.

In this paper, we take a step toward automating this methodology. Specifically, we propose an algorithm to estimate the degree to which two books share their typographical patterns, which is applicable to the analysis of roman and italic typefaces independently. Our approach adapts image processing techniques originally developed for unsupervised OCR post-processing~\cite{belzarena2025} to extract character prototypes from each book. It then measures typeface similarity by comparing these prototypes, applying an \emph{a contrario} analysis to ensure the statistical significance of the results. This analysis takes into account the complexity of the ground situation, where there is no equivalence between``same printer'' and ``same typeface'', typically due to type sharing across printers through trade or common sourcing. Our method, therefore, focuses on the degree of typeface similarity and stops short of attribution, which we leave to the expert.

By design, our method produces \emph{interpretable} and \emph{actionable} results. They enable efficient hypothesis generation, allowing experts to focus on assessing their validity; they also provide a global view of the relationships between editions. As our method does not rely on the availability or production of any training data, it enables philologists and other specialists to apply their expertise on an unprecedented scale.

The contributions of this work can be summarized as follows.
\begin{itemize}
    \item
    We propose a method to estimate the proximity of typefaces between books within a corpus, with a framework to statistically interpret measurements. Building on previous work in automatic printed character modeling, it is designed to meet the requirements of philological practice. In particular, it is able to detect italic characters and process roman and italic types independently. The code is available at \url{https://github.com/diegobelzarena/sueltas-typography}.
    \item
    We conduct experiments on two corpora of 17th-century \emph{sueltas}. Detailed philological analyses are performed on situations of interest identified by our method's results. This leads us to multiple discoveries, including the new attribution of several chapbooks to known printers, as well as the revision of some erroneous associations. We corroborate all new findings by following the traditional methodology.
    \item
    The scale at which this analysis is conducted is moderate (about 100 documents), but already far above the capability of a human expert inspection, which would have required about 5000 pairwise and very detailed comparisons.  Due to this positive result, we are in a position to scale up the process by acquiring thousands of scanned chapbooks from several European libraries.
\end{itemize}

This article is organized as follows. \Cref{sec:relatedwork} presents academic works that tackle similar issues or are built around similar ideas. \Cref{sec:method}  describe our method. \Cref{sec:experiments} provides all context and details concerning our experiments. \Cref{sec:discussion} reports the philological analyses, including the new findings and their justifications. Finally, \Cref{sec:conclusion} draws conclusions about our work.

\section{Related work} \label{sec:relatedwork}

The tasks of automatic localization and authorship attribution of historical documents have spurred research into computational approaches to \emph{paleography} (for manuscripts) and \emph{bibliography} (for printed works) \cite{kesmont2017,sommerschield2023}.

From the perspective of document image analysis and paleography, several systems have aimed to model character-level variation. Aiolli \etal~\cite{aiolli1999} proposed an early framework for paleographic inspection of medieval manuscripts based on internal redundancy within each document. The method requires extensive manual character segmentation, which limits scalability to large corpora. More recently, Vlachou-Efstathiou \etal~\cite{malamatenia2024} adapted a generative deep-learning document analysis model \cite{the-learnable-typewriter} for the comparison of Gothic manuscript subtypes. Although effective in capturing handwriting variability, best results are achieved through supervised training on labeled text lines.

In the field of digital bibliography focused on early modern print, Ryskina et al.~\cite{ryskina2017} used orthographic and spacing features to attribute compositors in Shakespeare's First Folio (1623). A closer parallel to our work is the approach by Vogler et al.~\cite{vogler2023}, who detect and match uniquely damaged sorts to identify printers in early modern English books—an approach in line with long-standing bibliographical practice, particularly in the study of 17th-century Spanish chapbooks~\cite{wilson1975}. However, this method is limited by its reliance on the accidental presence of damaged types, making it unsuitable for large-scale, generalizable comparison of the underlying shapes of complete round and italic typefaces.

Similarly, Vogler et al.~\cite{vogler2024} proposed a novel computational method to analyze running titles in early modern books. By clustering these titles, they track deviations from expected printing patterns, offering insights into variations in individual book production. However, this method focuses on a single type feature within a book and does not align with our goals, as our approach seeks to compare full typefaces across different books and texts to identify printers across large collections.

\section{Method} \label{sec:method}

We propose a fully automatic method for estimating the proximity of typefaces between printed documents in a given corpus, which does not rely on annotated training data and enables large-scale processing. Although it can be generalized to many corpora, we were motivated by the study of the 17th-century \emph{sueltas}; in this context, we refer to a document as a \emph{book}.

Roman and italic types are processed separately. Given a choice of typographic style, the method provides, for each pair of books, a statistically significant lower bound $\nhat_1$ of \textbf{the number of symbols for which the books use the same type}. This bound can be interpreted as follows: ``while some resemblances of typefaces may arise by chance, the observed similarities would be unlikely to occur unless the two books shared the same type for at least $\nhat_1$ symbols''.

Our approach follows a two-part methodology: measurement and statistical interpretation. Leveraging existing algorithms, the initial phase extracts character prototypes from each book to compute cross-book distances. We then introduce a novel application of the \emph{a contrario} framework to interpret these measurements.

We write $\Bcal$ for the set of books (\ie the corpus) and $\Symbs$ for the set of symbols (\ie the alphabet).

\subsection{Measuring typeface distances} \label{ssec:measurements}

The first step of the method is to compute, for each pair of books $b_1,b_2\in\Bcal$ and each symbol $s\in\Sigma$, a number $d_s(b_1,b_2)\in\R_{>0}$ which should be close to zero if and only if the type of $s$ in the chosen typographic style (roman or italic) is the same for the two books. This step is implemented as a sequential combination of established image processing techniques and algorithms.

\begin{figure}[ht]
    \centering
    \subfloat[]{\label{subfig:method1}\includegraphics[height=3.0cm]{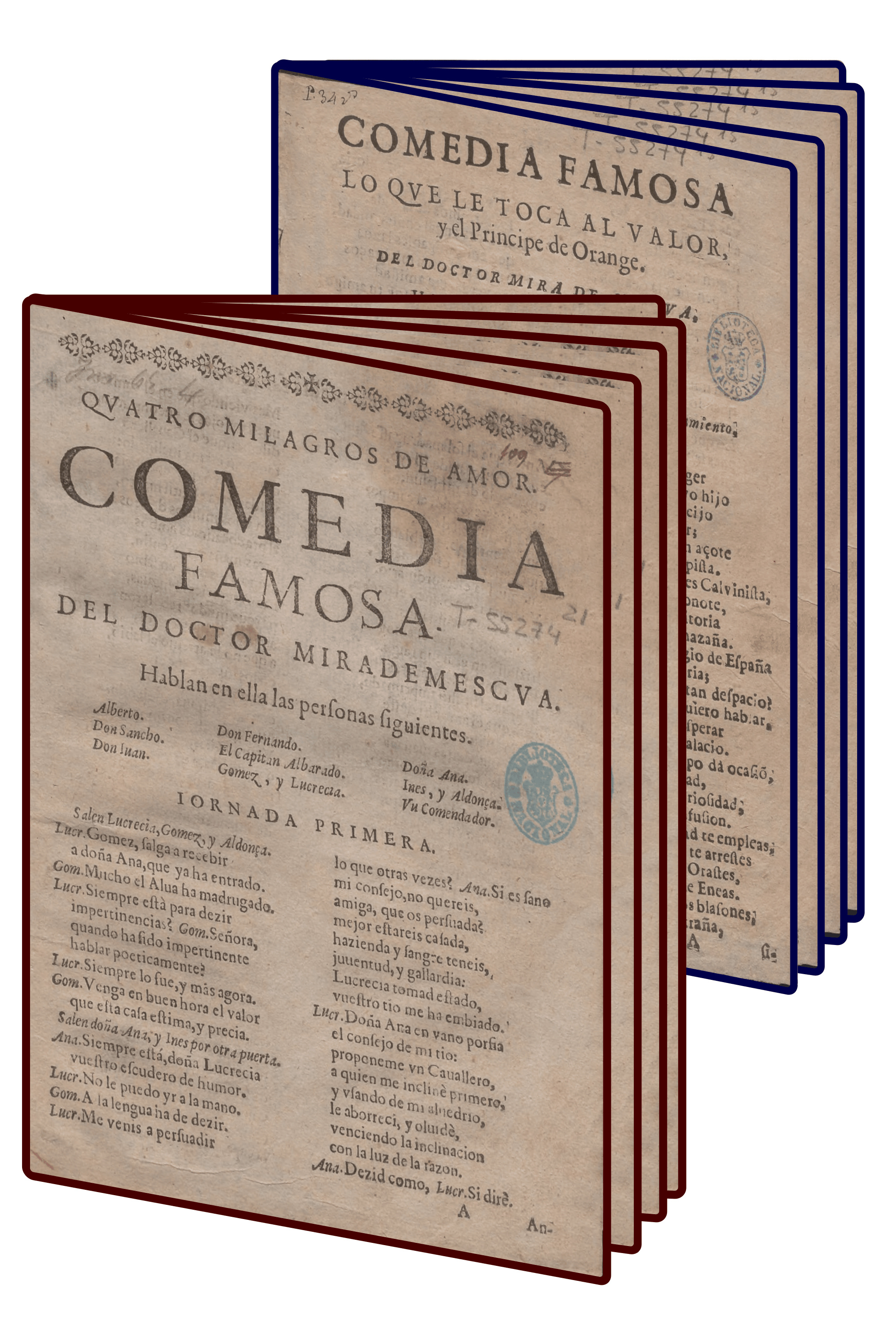}}
    \quad
    \subfloat[]{\label{subfig:method2}\includegraphics[height=3.0cm]{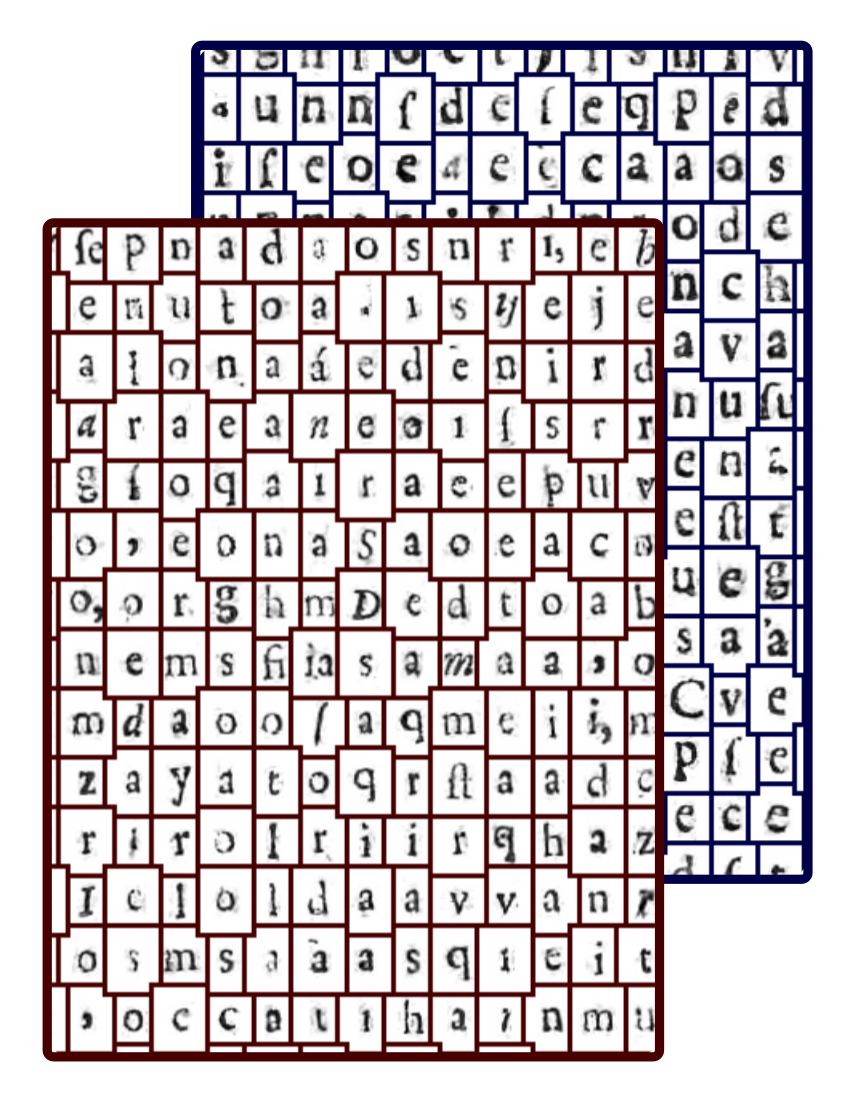}}
    \quad
    \subfloat[]{\label{subfig:method3}\includegraphics[height=2.8cm]{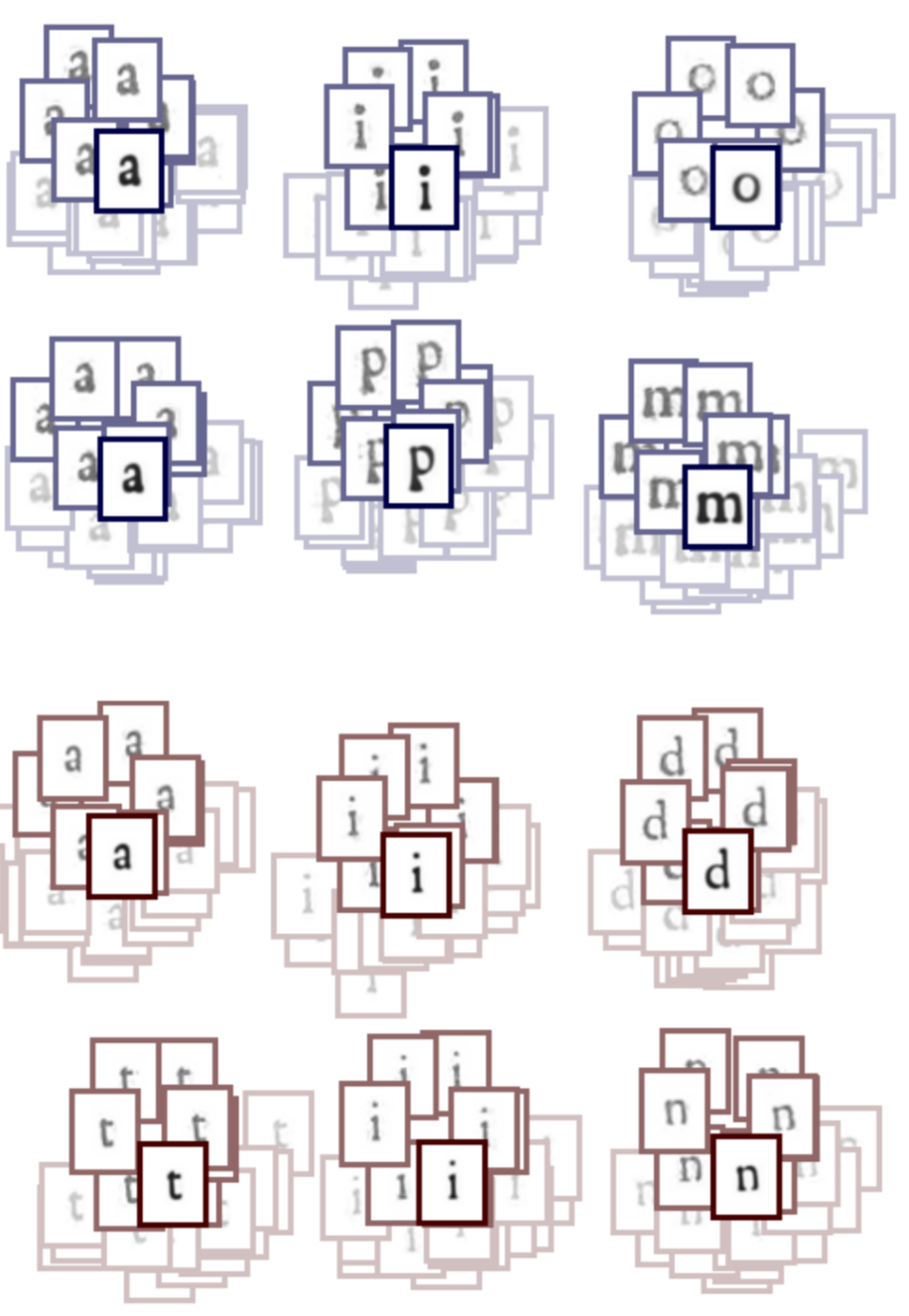}}
    \quad
    \subfloat[]{\label{subfig:method4}\includegraphics[height=3.0cm]{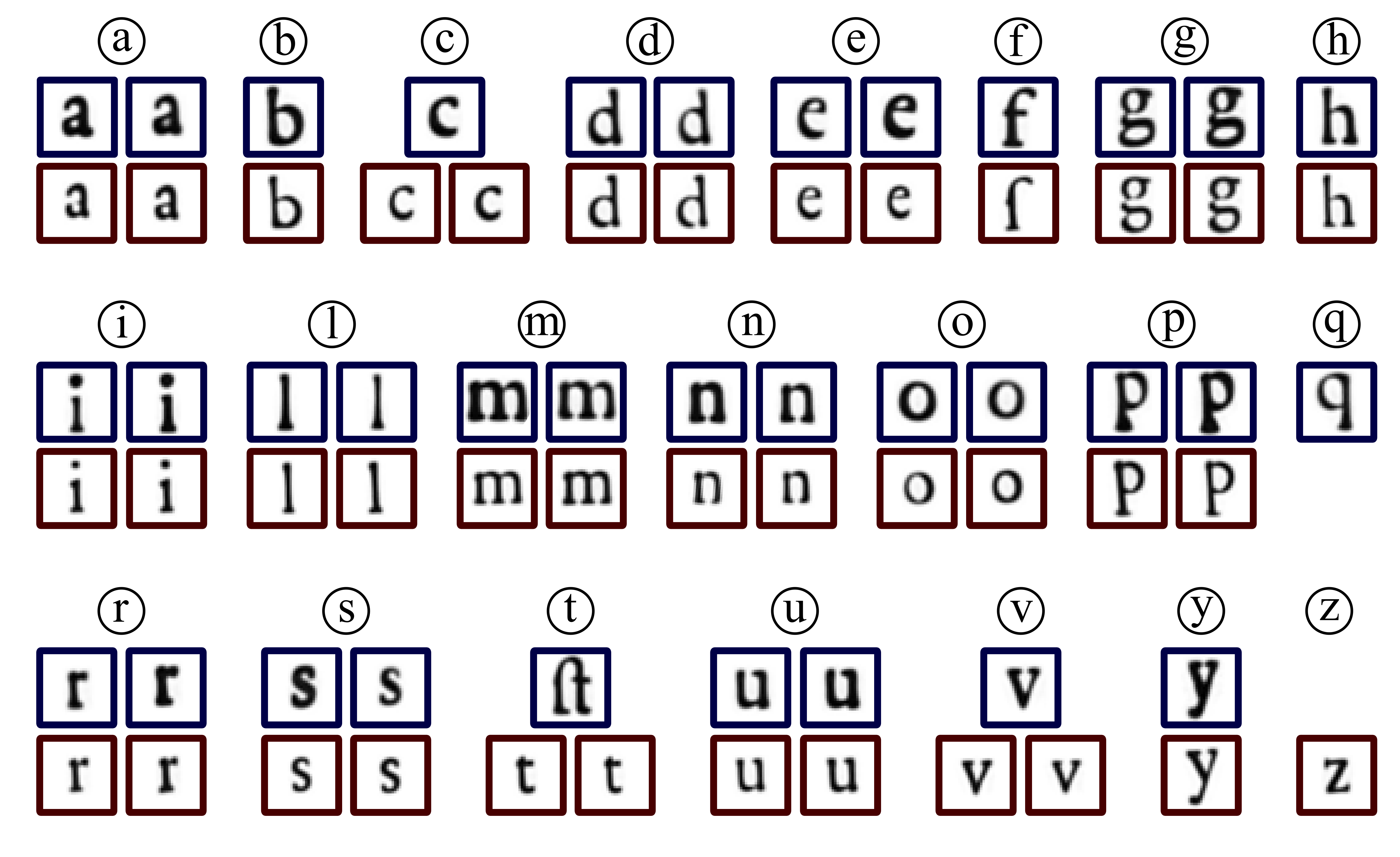}}
    \caption{First step of the method: book-specific character prototype extraction. Each book is processed independently. \protect\ref{subfig:method1} Input, an entire book/\textit{suelta}. \protect\ref{subfig:method2} Corrected character bounding boxes extracted with OCR, of the chosen type (roman or italic). \protect\ref{subfig:method3} Character clustering. \protect\ref{subfig:method4} Book-specific character prototypes.} 
    \label{fig:method}
\end{figure}

\paragraph{Character extraction}

We apply an off-the-shelf OCR method, namely CharNet~\cite{xing2019charnet}, to extract labeled character bounding boxes. In parallel, roman \vs italic style classification is performed at the word level based on the difference between slant and skew: the slant angle of a word is estimated as the principal orientation of the structure tensor over a word's bounding box, and compared with the page's skew angle, which we estimate semi-globally on tiles using the Fourier method~\cite{postl1986}. Only character boxes matching the desired font type are retained for subsequent processing.

\paragraph{Character prototypes}

Following the method proposed by Belzarena~\etal~\cite{belzarena2025}, specifically the steps in Sections 3.1--3.3, retained bounding boxes are corrected, and characters are cropped into $40\times32$ images (see~\cref{subfig:method2}) then jointly aligned and clustered in an unsupervised manner. This produces a number of unlabeled clusters (see~\cref{subfig:method3}), which can be noisy for various reasons, such as low-quality images, incorrect labeling, improper segmentation, or imperfect processing.

To confidently assign a character label to each cluster while removing outliers, the clusters are filtered according to two conditions: we impose lower limits 1) $C_\mathrm{min}$ on the number of elements in the cluster (following our observations that low-quality clusters tend to be smaller) and 2) $P_\mathrm{min}$ on the biggest percentage of elements labeled with the same symbol by the OCR. We attribute this majority label to each of the filtered clusters.

At this stage, the elements of each cluster are aligned character images of a consistent type; the mean image (a vector in $\R^{40\times32}$) is the \emph{character prototype} of the cluster.

\paragraph{Distances between prototypes}

Given a book $b\in\Bcal$ and a symbol $s\in\Sigma$, denote by $\Tcal_s(b)\subseteq\R^{40\times32}$ the set of character prototypes extracted in the previous step and labeled $s$ (see \cref{subfig:method4}). For a pair of books $b_1,b_2\in\Bcal$, let us define
\begin{equation}
    \Sdef(b_1,b_2) \coloneqq \{s\in\Sigma\mid \Tcal_s(b_1)\neq\emptyset,\; \Tcal_s(b_2)\neq\emptyset\} \,.
\end{equation}
For each $s\in\Sdef(b_1,b_2)$, we compute the ``typeface distance'' between $b_1$ and $b_2$ at $s$ as the minimal cosine distance between prototypes:
\begin{equation} \label{eq:cosinedist}
    d_s(b_1,b_2) \coloneqq \min_{(t_1,t_2)\in\Tcal_s(b_1)\times\Tcal_s(b_2)} \cosdist(t_1,\ICA_{t_1}(t_2)) \,
\end{equation}
where $\ICA_{t_1}(t_2)$ is the euclidean transformation of image $t_2$ registering it onto the reference image $t_1$, computed with the Inverse Compositional Algorithm~\cite{baker2001equivalence,briand2018improvements}.

The rationale for taking the minimum cosine distance is that multiple type variants may be used for the same symbol in a book. For example, a compositor might begin using regular-sized types for the first pages, but later switch to smaller types to fit the book within a convenient page count \cite{wilson1975}. As a result, two books from the same printing house may contain different combinations of types—one may use only regular-sized `o's, while another may feature both regular-sized and smaller-sized `o's. Using a minimum avoids unfair penalization of typographical discrepancies arising from compositional choices and captures the best match between the two books.

\subsection{Lower bounds for the number of shared types} \label{ssec:acontrario}

Consider two random variables $b_1$ and $b_2$ modeling the selection of a pair $(b_1,b_2)$ of books with values in $\Bcal$. Then $\Sdef(b_1,b_2)$ is a random subset of $\Symbs$, and $d_s(b_1,b_2)\in\R_{\geq0}\cup\{+\infty\}$ is a random variable (we set $d_s=+\infty$ when $s\notin\Sdef(b_1, b_2)$).

Let $\Symbs_0(b_1,b_2)$ (resp.\ $\Symbs_1(b_1,b_2)$) denote the set of symbols for which $b_1$ and $b_2$ do not (resp.\ do) use the same type; $\Symbs_0$ and $\Symbs_1$ form a random partition of $\Sdef$. Note that $\abs{\Sdef}$ is observable, while $\abs{\Symbs_0}$ and $\abs{\Symbs_1}$ are not. The goal of this step is to output an estimate of the latter quantity $\abs{\Symbs_1(b_1,b_2)}$ using the observation of $d_s(b_1,b_2)$, with clear statistical guarantees.

We make the following assumption of conditional independence; conditioning is crucial, since without it the events pertaining to different symbols are obviously correlated with each other through the possibility that the printers of $b_1$ and $b_2$ are related in some form or another.

\begin{ass} \label{ass:independence}
    For $s\in\Symbs$, let $\Fcal_s$ denote the $\sigma$-algebra generated by the variable $d_s$ and the events $s\in\Symbs_0$, $s\in\Symbs_1$. For each $s$, we assume that the $\Fcal_{s'}$ for $s'\neq s$ are independent of $\Fcal_s$ conditionally to $(s\in\Symbs_0)$.
\end{ass}

We fix a probability threshold $\alpha\in(0,1)$ and define, for $s\in\Symbs$, the ``background'' quantiles $d_{s,\alpha}\in\R_{\geq0}$ of level $\alpha$ by:
\begin{equation} \label{eq:quantile}
    \Prob(d_s(b_1,b_2)<d_{s,\alpha}\mid s\in\Symbs_0(b_1,b_2)) = \alpha \,.
\end{equation}

Now consider the variables
\begin{equation}
    K_\alpha(b_1,b_2) \coloneqq \sum_{s\in\Sdef(b_1,b_2)} \indic_{d_s(b_1,b_2)<d_{s,\alpha}}
\end{equation}
and $K_{i,\alpha}(b_1,b_2)$ defined similarly for $i\in\{0,1\}$ with $\Symbs_i$ in place of $\Sdef$; the realizations of $K_{0,\alpha}$ and $K_{1,\alpha}$ are hidden to the observer, but $K_\alpha=K_{0,\alpha}+K_{1,\alpha}$ is observable.

We want to find the largest $n_1\in\Z_{>0}$ for which we can reject the hypothesis $(\abs{\Symbs_1}<n_1)$. To this end, given  the observed values $k_\alpha$ of $K_\alpha$ and $n$ of $\abs{\Sdef}$, we estimate a $p$-value defined by the probability of $(K_\alpha\geq k_\alpha)$ conditionally to the event $A_{n,n_1}\coloneqq(\abs{\Sdef}=n)\wedge(\abs{\Symbs_1}<n_1)$. Writing by $B_{n,n_1}\coloneqq(\abs{\Sdef}=n)\wedge(\abs{\Symbs_1}=n_1)$, we have
\begin{equation} \label{eq:K0bound}
     \begin{aligned}
    \Prob(K_\alpha\geq k_\alpha\mid A_{n,n_1})     & \leq \max_{n_1'<n_1} \Prob(K_\alpha\geq k_\alpha\mid B_{n,n'_1}) \\
        & \leq \max_{n_1'<n_1} \Prob(K_{0,\alpha}\geq k_\alpha-n'_1\mid B_{n,n'_1})
    \end{aligned}
\end{equation}
where the first inequality results from $A_{n,n_1}=\coprod_{n'_1<n_1}B_{n,n'_1}$ and the second results from $(K_\alpha\geq k_\alpha)\wedge(\abs{\Symbs_1}=n'_1)\subseteq(K_{0,\alpha}\geq k_\alpha-n'_1)$. This upper bound is made tractable by the following result.

\begin{lem} \label{lem:binomial}
    The conditional variable $(K_{0,\alpha}\mid B_{n,n'_1})$ follows the binomial law $\Bin(n-n'_1,\alpha)$.
\end{lem}

\begin{proof}
    Since $B_{n,n'_1}$ is the disjoint union of the events $B_{S_0,S_1}\coloneqq(\Symbs_0=S_0)\wedge(\Symbs_1=S_1)$ for $S_0$ and $S_1$ running over disjoint subsets of $\Symbs$ of sizes $n-n'_1$ and $n'_1$, it is enough to prove that $(K_{0,\alpha}\mid B_{S_0,S_1})\sim\Bin(n-n'_1,\alpha)$ for any such $S_0,S_1$.

    By \Cref{ass:independence}, for all $s\in S_0$, one has
    \begin{equation}
        d_s \indep \left((d_{s'})_{s'\neq s},\; \Symbs_0\setminus\{s\}=S_0\setminus\{s\},\; \Symbs_1=S_1\right)
    \end{equation}
    conditionally to $(s\in\Symbs_0)$, hence $d_s \indep (d_{s'})_{s'\neq s}$ conditionally to $(s\in\Symbs_0) \wedge (\Symbs_0\setminus\{s\}=S_0\setminus\{s\}) \wedge (\Symbs_1=S_1)$, with this latter condition being exactly $B_{S_0,S_1}$. Therefore, $(K_{0,\alpha}\mid B_{S_0,S_1})=\sum_{s\in S_0}(\indic_{d_s<d_{s,\alpha}}\mid B_{S_0,S_1})$ is a sum of $n-n'_1$ mutually independent Bernoulli variables which, by \Cref{eq:quantile}, all have parameter $\alpha$.
\end{proof}

When repeating the analysis for a range $\{\alpha_i\}_{i\in I}$ of probability thresholds, a refined $p$-value can be considered: given observations $(k_{\alpha_i})_{i\in I}$,
\begin{equation} \label{eq:multipleAlphas}
    p(n_1, n, (k_{\alpha_i})_{i\in I})
    \coloneqq \Prob\left(\bigwedge\nolimits_{i\in I}(K_{\alpha_i}\geq k_{\alpha_i})\,\middle|\, A_{n,n_1}\right)
    \leq \min_{i\in I} \Prob(K_{\alpha_i}\geq k_{\alpha_i}\mid A_{n,n_1}) \,.
\end{equation}

In the \emph{a contrario} framework \cite{desolneux2007}, we consider the total number $N\coloneqq\abs{\Bcal}(\abs{\Bcal}-1)/2$ of samples (pairs of books) which we analyze and apply a Bonferroni-type correction: we fix a threshold $p$ for the $p$-values such that the expected number $Np$ of $p$-extreme events is below a given threshold $\varepsilon$. Combining \Cref{eq:K0bound,lem:binomial,eq:multipleAlphas}, this leads us to output the following value:
\begin{equation} \label{eq:n1Output}
    \nhat_1(b_1,b_2) \coloneqq \min \{n_1\in\Z_{\geq0} \mid \min_{i\in I}\Prob_{Z\sim\Bin(n-n_1,\alpha_i)}(Z\geq k_{\alpha_i}-n_1)\geq\varepsilon/N \} \,.
\end{equation}

\subsubsection{Estimation of distance thresholds.}

Making observations of $k_\alpha$ in \Cref{eq:n1Output} requires estimating $d_{s,\alpha}$. Without the knowledge of $\Sigma_0$, we make a second assumption, which says that $d_s$ is lower when $s\in\Symbs_1$ than when $s\in\Symbs_0$ (\ie the distances indeed indicate proximity).

\begin{ass}
    Writing $F_X$ for the cumulative distribution function of a real variable $X$, we assume that $F_{d_s\mid s\in\Symbs_0}\leq F_{d_s\mid s\in\Symbs_1}$; equivalently, $F_{d_s\mid s\in\Symbs_0}\leq F_{d_s\mid s\in\Sdef}$.
\end{ass}

We then take a conservative approach by taking $\dhat_{s,\alpha}$ as the order statistic of level $\alpha$ of the sample $\{d_s(b_1,b_2)\mid b_1,b_2\in\Bcal,\; b_1\neq b_2,\; s\in\Sdef\}$, thus giving
\begin{equation} \label{eq:quantileInequality}
    \Prob(d_s<\dhat_{s,\alpha}\mid s\in\Symbs_0) \leq\Prob(d_s<\dhat_{s,\alpha}\mid s\in\Sdef)\approx\alpha \,.
\end{equation}

\section{Experimental setup} \label{sec:experiments}

\subsection{Two datasets of sueltas}

 We apply our method to two corpora of \emph{sueltas} from the 17th century, both of which are subsets of the materials being studied as part of the ISTAE project\footnote{\url{https://istae.uv.es/}}. The first corpus, which is the larger one, corresponds to the use case of large-scale philological analysis; we use it to evaluate our method as a guiding tool for expert analysis. The second corpus serves as a test of our method's usefulness in providing answers to a specific situation of interest identified by experts. Both corpora are available for download at \url{https://doi.org/10.5281/zenodo.18672988}.

\paragraph{First corpus: broad evaluation} \label{ssec:corpus1}

The corpus $\Bcal_1$ comprises 88 \emph{sueltas} supplied by the Biblioteca Nacional de España (BNE), selected without particular filtering based on availability to the authors. The bibliographical data for these \emph{sueltas} is missing. Despite this, the printer attribution is known to philologists for 14 out of the 88 chapbooks: by traditional methods, 2 \emph{sueltas} have been attributed to Lucas Antonio de Bedmar, 4 to Simón Fajardo, 3 to Francisco de Lyra, 2 to Pedro Gómez de Pastrana, and 3 to Manuel de Sande.

\paragraph{Second corpus: focused case study} \label{ssec:corpus2}

The corpus $\Bcal_2$ consists of 21 chapbooks: 5 loose sheets from the Bibliothèque Nationale Universitaire of Strasbourg (BNU) and 16 from the Biblioteca Nacional de España (BNE). As in the first corpus, bibliographical data is missing for these \emph{sueltas}, but most of them have been linked to one of two editorial families: 14 of those from the BNE have been linked to the workshop of Francisco de Lyra, and the 5 from the BNU to Nicol{\'a}s Rodr{\'i}guez de {\'A}brego. In addition, the remaining two from the BNE are suspected to be from the same printers.

This corpus corresponds to a question recently raised about the validity of those attributions. The first family---Lyra's---was established on the basis of characteristics including the compositional pattern, or graphic architecture \cite{cruickshank1989,vega2001,vega2002}. However, it was discovered that the second family had the same graphic architecture as the first's \cite{casariego2025}. One notable difference was found: Rodr{\'i}guez de {\'A}brego's italics were identified as \emph{cicéro cursive} by Robert Granjon, while one of the books from Lyra used Haultin's \emph{médiane italique grasse} \cite{cruickshank1989}. We therefore propose to experiment on this small corpus to revisit the attributions to printing houses, focusing specifically on the analysis of italic types, thereby testing the applicability of our method against a problem of current relevance.

\subsection{Implementation details} \label{ssec:implementation}

Prior to processing, the input document images are rescaled to a common resolution of 150dpi (we assume the original scanning resolution is known). The character prototype extraction step uses the source algorithm of Belzarena \etal\cite{belzarena2025}. In all our experiments, the \emph{a contrario} analysis is performed with the same parameters: $\varepsilon=0.01$ and $\{\alpha_i\}_{i\in I}=\{0.1\cdot 2^k, 0.15\cdot 2^{k+1}\mid k\in\Z, -5\leq k\leq 2\}$.

For the experiment on the italic types in the smaller corpus $\Bcal_2$, we apply the method with a slight variation from the description in \Cref{sec:method}, limited to estimating the quantiles $\dhat_{s,\alpha}$. As it stands, the very limited number of types at play (namely two: Haultin's \emph{médiane italique grasse cursive} and Granjon's \emph{cicéro cursive}) makes it so that, for each $s\in\Symbs$, the pairs $b_1,b_2\in\Bcal_2$ for which $s\in\Symbs_1(b_1,b_2)$ constitute a significant proportion of those for which $s\in\Sdef(b_1,b_2)$. Consequently, the inequalities in \Cref{eq:quantileInequality} are very wide, which hampers the efficiency of the method. To address this, we instead use the quantiles $\dhat_{s,\alpha}$ computed on the corpus $\Bcal_1$: this translates to considering that the background model of $d_s\mid s\in\Symbs_1$ is the same as for $\Bcal_1$, which makes sense since both corpora come from the same body of material.

Our experiments were carried out on  Nvidia Tesla P100 GPU for CharNet, and on 16  CPU cores for the rest of the algorithm, taking around 15 minutes per chapbook (each book consisting of 30 to 40 pages). Running time for the second part of the algorithm (\cref{ssec:acontrario}) is several orders of magnitude below that of the first part (\cref{ssec:measurements}), so the total running time scales linearly with the size of the corpus.

\subsection{Visualization of the outputs} \label{ssec:vis}

The main output of our method consists of the data of $\nhat_1(b_1,b_2)$ for all pairs of books $b_1, b_2\in\Bcal^2$, \ie, a lower bound on the number of similar
types between each pair of books. Additionally, any relevant intermediary values can be made available to the expert end-user; we found the following to be particularly useful:
\begin{itemize}
    \item
    the number $n(b_1, b_2)\coloneqq\abs{\Sdef(b_1,b_2)}$ of analyzable symbols;
    \item
    the list $\{s\in\Sdef(b_1,b_2)\mid d_s(b_1,b_2)<d_{s,\alpha_i}\}$ of candidate symbols (of which the cardinality is $k_{\alpha_i}$), for $i\in I$ corresponding to the maximal output $\nhat$ (see \cref{eq:n1Output});
    \item
    the character prototype images $\Tcal_s(b_1)$ and $\Tcal_s(b_2)$, for $s$ in the above list.
\end{itemize}

Experiments are performed in three settings: roman and italic types for $\Bcal_1$, and italic types for $\Bcal_2$. In the following, we write $\nhat_1^{(1,\rmm)}$, $\nhat_1^{(1,\itt)}$, or $\nhat_1^{(2,\itt)}$ respectively when referring to $\nhat_1$ in a specific setting, and similarly for other values.

Beyond the raw output data, we offer two modes of visualization, detailed below: a matrix and a graph. In line with our objective of enabling large-scale human analysis, these aim to provide a perceptual view of the proximity relationships — according to our method — between the books in the corpus. To take into account the restrictions imposed by the character extraction step, we encode proximity as a weight $w(b_1,b_2)\coloneqq\nhat_1/n\in[0,1]$. It is important to be reminded that, by the nature of the problem, such a notion of proximity inherently includes an element of arbitrariness, as opposed to the precise interpretability of the output value $\nhat_1$.

\paragraph{Matrix representation}

Given a particular ordering of $\Bcal$, the output data of $(\nhat_1(b_1,b_2))_{(b_1,b_2)\in\Bcal^2}$ can be represented as a square matrix. To enhance visual interpretability, the ordering is optimized to maximize the matrix's ``block-diagonality''. This is achieved automatically through optimal leaf ordering in UPGMA hierarchical clustering \cite{upgma}. For our experiment on $\Bcal_2$, the metric provided to the UPGMA algorithm is $1-w^{(2,\itt)}$. When both italics and roman types are performed on the same corpus, it is best to represent results of both experiments with the same ordering; hence, we use the aggregate metric $1-(w^{(1,\rmm)}+w^{(1,\itt)})/2$ for $\Bcal_1$. The resulting matrices are shown in \Cref{fig:matrix1,subfig:matrix2}.

\paragraph{Graph representation}

We plot a planar graph whose vertices are the books in the corpus and whose edges have opacity proportional to their weight $w$. The graph's geometry, \ie the position of the vertices, is automatically set with UMAP~\cite{umap} using the same metric as with the matrix representation above. Since UMAP is fast, the user can tune its parameters \texttt{n\_neighbors} and \texttt{min\_dist}, as results on different corpora may be best visualized using different values (especially for \texttt{n\_neighbors}); we used $10$ and $1$ respectively for $\Bcal_1$, and $20$ and $1$ for $\Bcal_2$. The resulting graphs are shown in \Cref{fig:graph1,subfig:graph2}.

\subsection{Validation of results} \label{ssec:validation}

By definition, the automatic and quantitative evaluation of our method is unfeasible. Indeed, it is not invalidating to find a high---let alone non-zero---value of $\nhat_1(b_1,b_2)$ with two books $b_1$ and $b_2$ from different printing workshops, as different printers bought types from common manufacturers and would even trade them among themselves. Conversely, finding $\nhat_1(b_1,b_2)=0$ for $b_1$ and $b_2$ of the same workshop does not bring about fundamental discredit, given that $\nhat_1$ is designed as a highly confident lower-bound of $\abs{\Symbs_1}$, and that a given printer could use different types over their career (or even in the same book!).

Consequently, our method should be judged by its ability to provide relevant and non-trivial philological information: any evaluation must be qualitative and rely on expert feedback. Our methodology consists of the following steps. First, our image processing team runs the method as described in \Cref{ssec:implementation}, and provides results in the formats detailed in \Cref{ssec:vis}. Then, our philology team, on the basis of prior knowledge and by applying traditional methods when deemed necessary, judges these results according to: 1) whether they impart, perceptually, an accurate or misleading picture of the relationships between \emph{sueltas}; 2) the added value provided by our method outputs with respect to the discovery of new information, \ie the quality of the hypotheses it generates.

\section{Results} \label{sec:discussion}

In this section, we describe the outcome of our experiments. We validate the soundness of the method's outputs by testing known bibliographical data and addressing potential biases related to image quality. Furthermore, we present cases where the application of our method has led to meaningful philological contributions: these include identifying several new printer attributions---both extending existing editorial families and creating new ones---as well as rectifying two erroneous attributions and confirming many others in light of recent doubts raised in the literature.

\subsection{First corpus: general aspects} \label{ssec:discussion1}

Let us first discuss the analysis of the bigger corpus $\Bcal_1$ through both its roman and italic letters. The matrix, in \Cref{fig:matrix1}, and the graph, in \Cref{fig:graph1}, illustrate the relationships between the 88 theater chapbooks, which have potentially varied origins and time periods.

\begin{figure}[ht]
    \centering
    \subfloat[Roman types.]{\includegraphics[width=0.45\columnwidth]{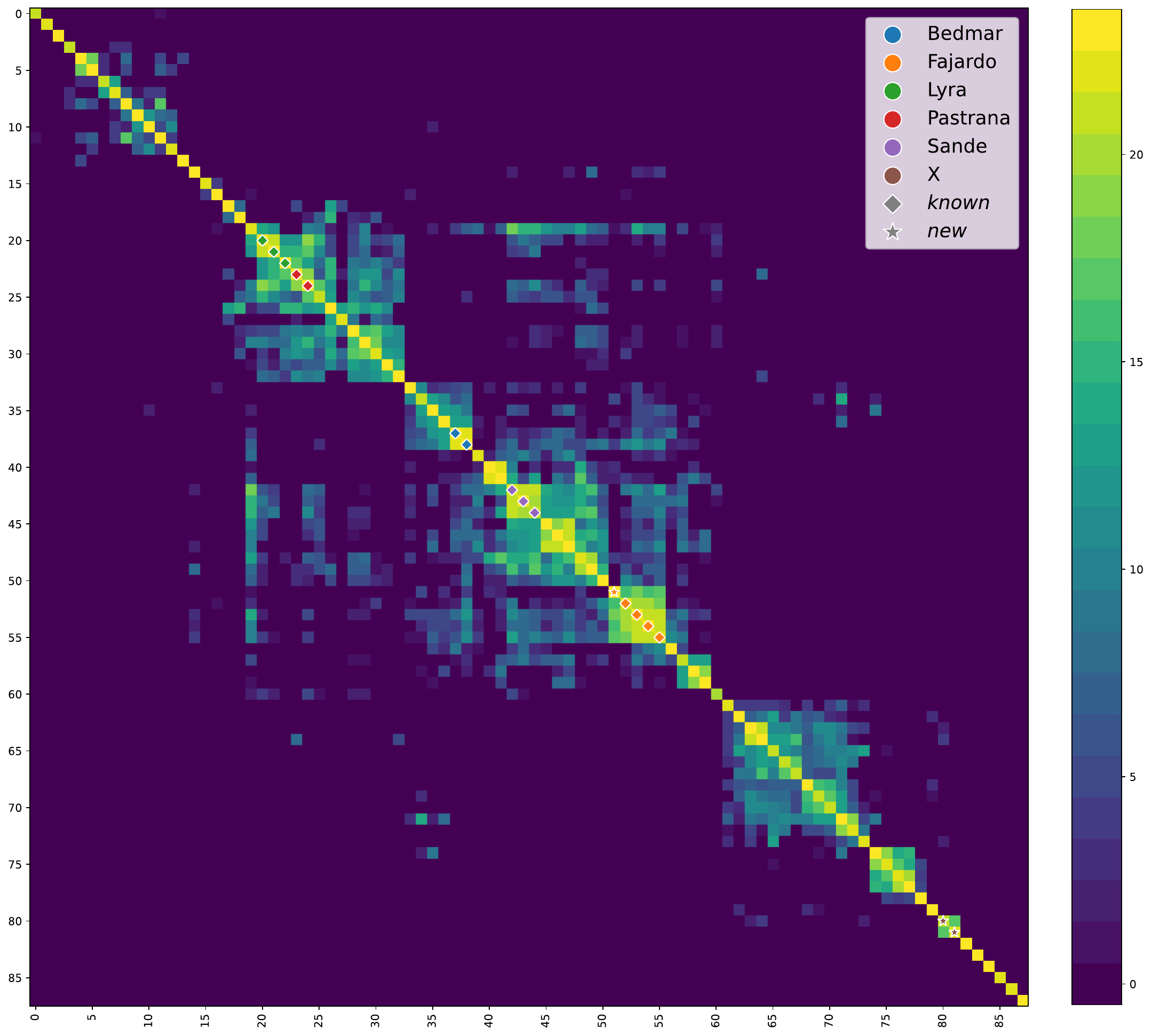}}
    \quad
    \subfloat[Italic types.]{\includegraphics[width=0.45\columnwidth]{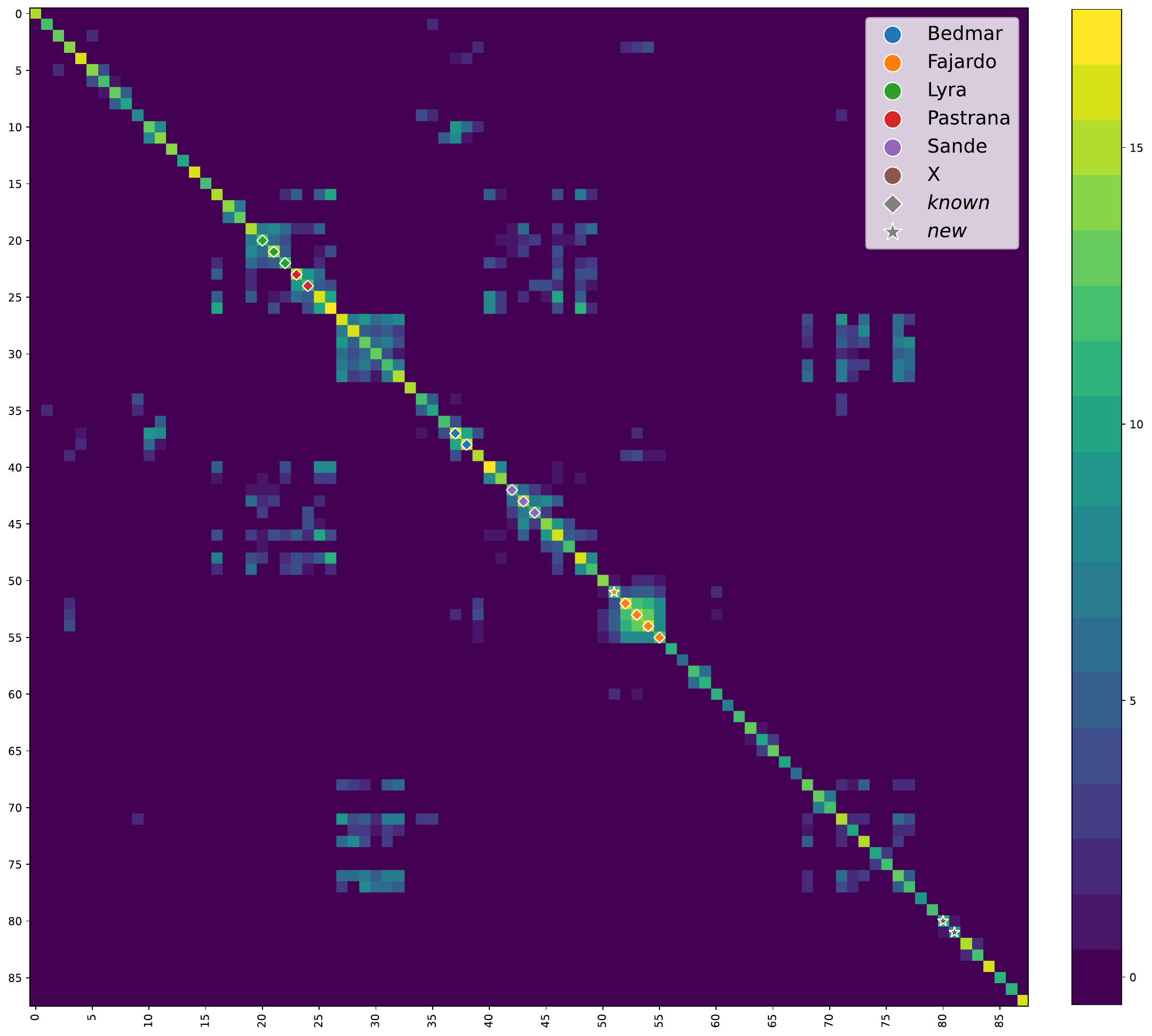}}
    \caption{Matrix representation of $\nhat_1$ on the corpus $\Bcal_1$ (\cref{ssec:corpus1}) ordered by hierarchical clustering; see \cref{ssec:vis}. The \emph{sueltas} with known printer attribution (including those which were hidden to serve as control), as well as those on which new knowledge has been gained thanks to our method (see \cref{ssec:discussion1}), are highlighted on the diagonal; printer `X' denotes a new family whose publishing house is yet unknown.}
    \label{fig:matrix1}
\end{figure}

\begin{figure}[ht]
    \centering
    \subfloat[Roman types.]{\label{subfig:graph1rm}\includegraphics[width=0.45\columnwidth]{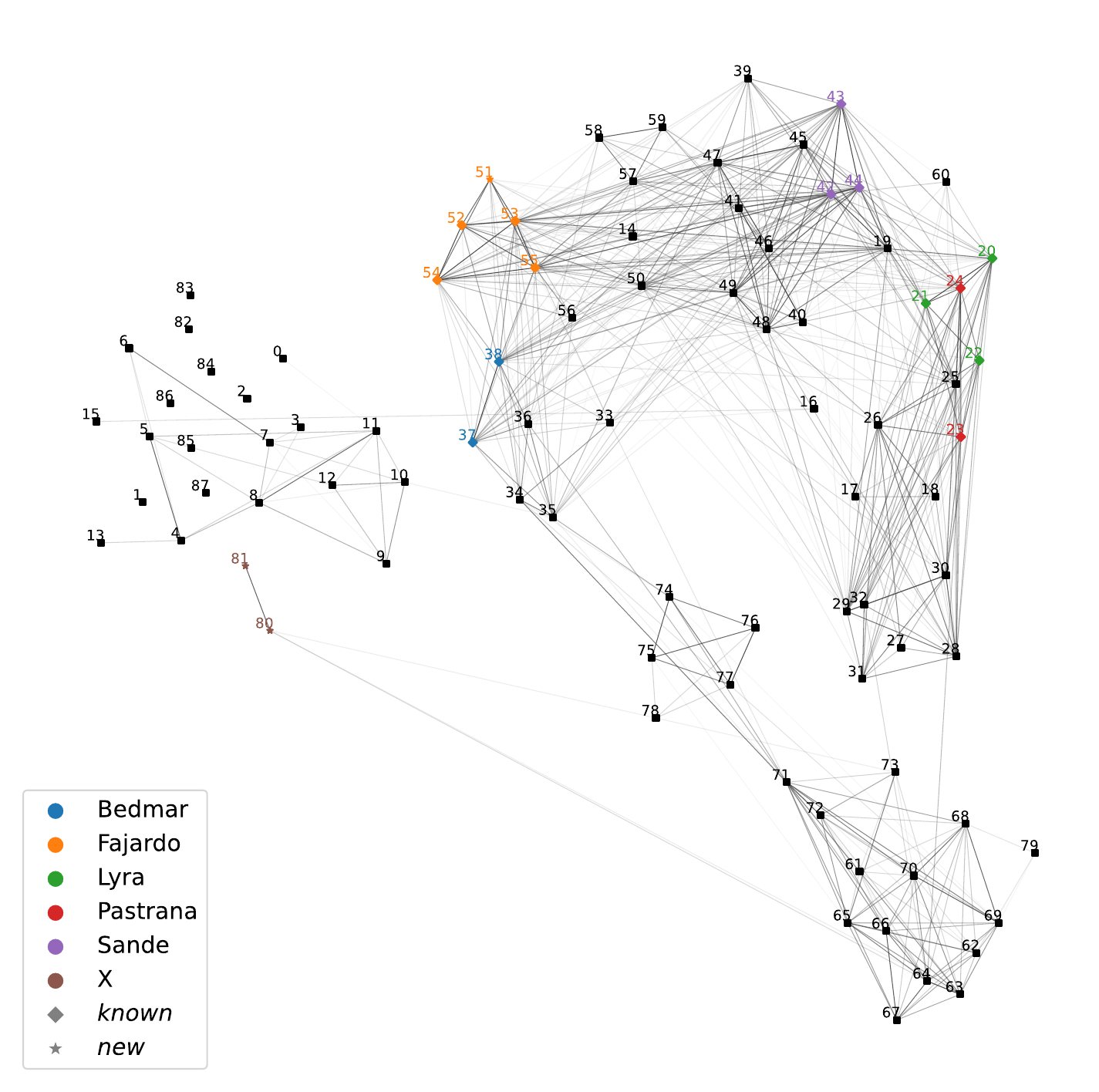}}
    \quad
    \subfloat[Italic types.]{\label{subfig:graph1it}\includegraphics[width=0.45\columnwidth]{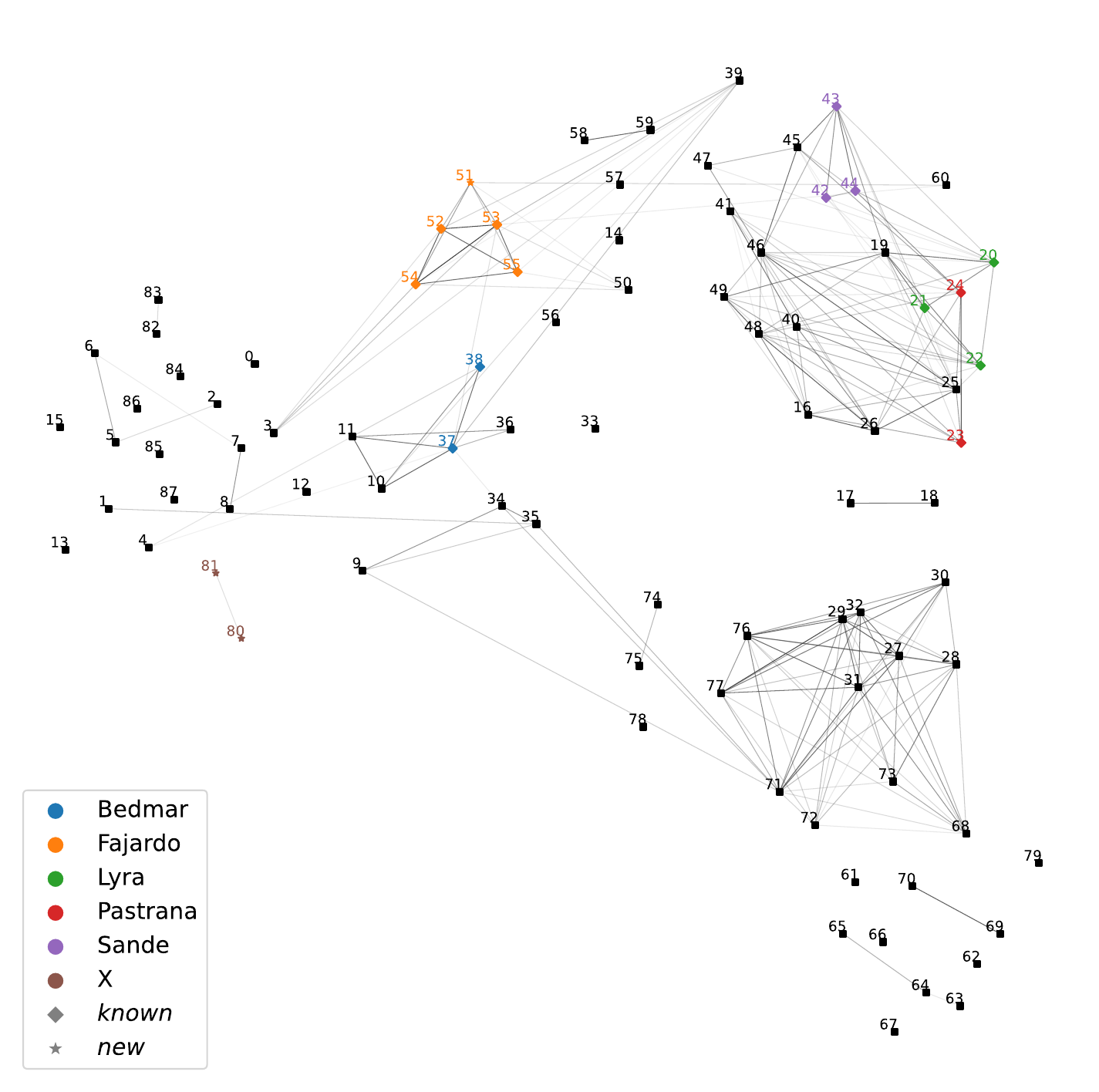}}
    \caption{Graph representation on the corpus $\Bcal_1$ (\cref{ssec:corpus1}); the numbering and highlighting of vertices match \cref{fig:graph1}.}
    \label{fig:graph1}
\end{figure}

\paragraph{Checking known relationships}

We first verify our method's results on the 14 \emph{sueltas} with known attribution: outside expert validation, these serve as the primary ground truth for our analysis. The automatic ordering of books arranges known editorial families into sequences of contiguous books; furthermore, members of the same families are indeed within the same blocks in the matrices, and within the same cliques in the graphs. \Cref{fig:hist1} displays the histogram of outputs values $w=\nhat_1/n$ for these 14 \emph{sueltas}, showing the extent of our method's discriminatory power. Quantitatively, this power can be assessed by treating $w=\nhat_1/n$ as a similarity score on two tasks over this labeled subset: a pairwise same-/different-printer classification over the 91 pairs (14 of them same-printer), and per-book retrieval of same-printer \emph{sueltas}. \Cref{tab:metrics} reports the ROC-AUC and mAP. In line with \Cref{ssec:validation}, these figures quantify the strength of the typographic signal on the few books for which printer labels are available; they are not a validation of attribution, and because printers shared and exchanged types, neither perfect separation nor a perfect score is expected or desirable.

\begin{figure}[ht]
    \centering
    \subfloat[Roman types.]{\label{subfig:hist1rm}\includegraphics[width=0.45\columnwidth]{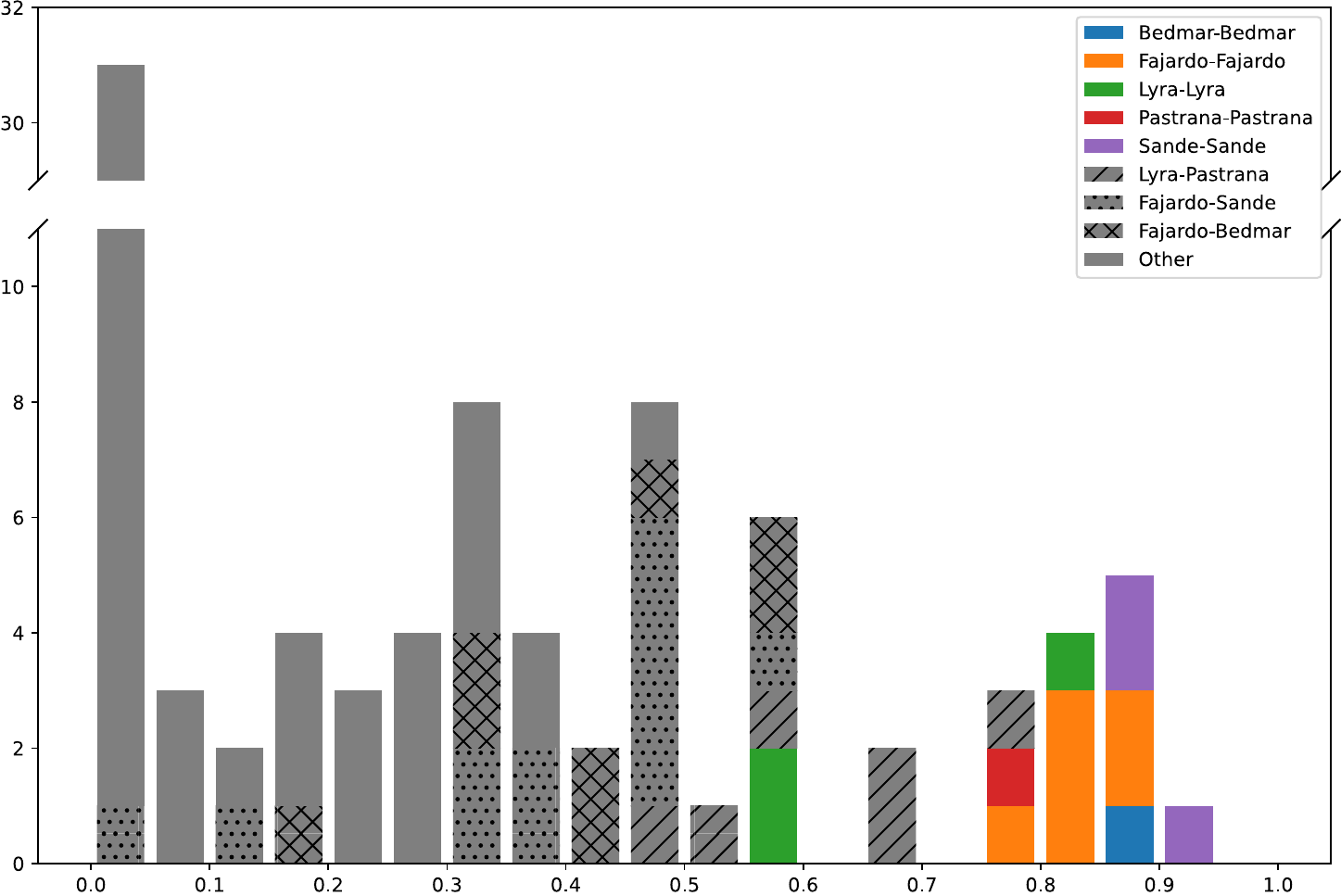}}
    \quad
    \subfloat[Italic types.]{\label{subfig:hist1it}\includegraphics[width=0.45\columnwidth]{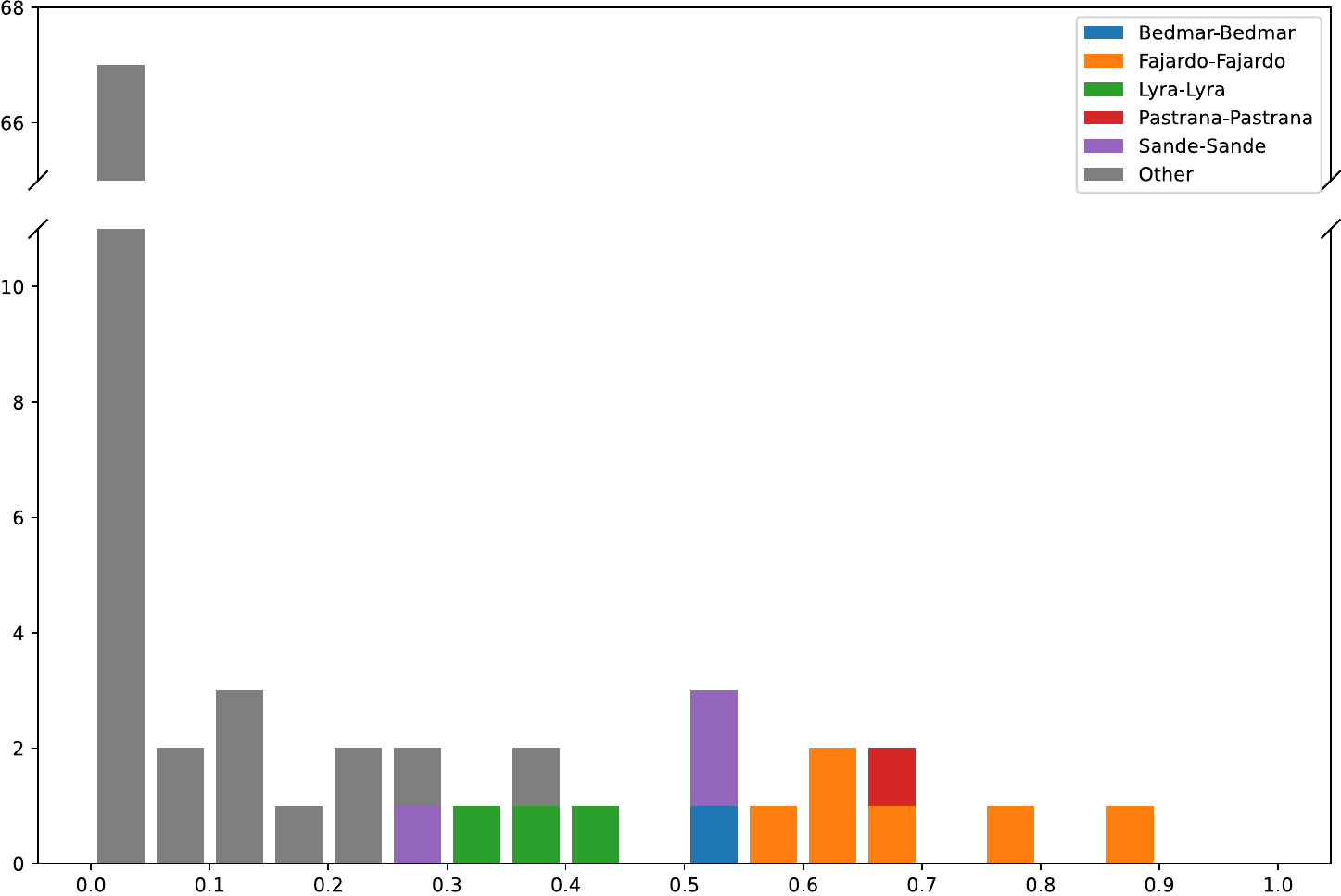}}
    \caption{Histograms of values of $w=\nhat_1/n$ on the subset of $\Bcal_1$ formed by \emph{sueltas} with known attributions.}
    \label{fig:hist1}
\end{figure}

\begin{table}[ht]
    \centering
    \captionsetup{skip=10pt}  
    \begin{tabular}{lcc}
        \toprule
         & AUC & mAP \\
        \midrule
        Roman types  & 0.989 & 0.89 \\ 
        Italic types  & 0.997 & 0.98 \\ 
        \bottomrule
    \end{tabular}
    \caption{Evaluation metrics for $w=\nhat_1/n$ on the subset of $\Bcal_1$ formed by \emph{sueltas} with known attributions (see \cref{fig:hist1}). ROC-AUC is computed for the pairwise same- \vs different-printer classification task and mAP for per-book retrieval. Consistent with the philological expectation that italic types are more discriminative than roman ones, both metrics are higher for italics. Note that a perfect score is neither expected nor desirable: see \cref{ssec:validation}.}
    \label{tab:metrics}
\end{table}

A few pairs of books from different printers have $\nhat_1$ reach significant values, even if not the highest, in the roman case. We find that it is not contradictory with the interpretation of $\nhat_1$ (which, as we remind the reader, does not statute on printer attribution but only on typeface similarity), as these cases consistently align with established philological knowledge. Indeed, it is documented that Lyra and Gómez de Pastrana shared typefaces~\cite[p.976]{cruickshank2004}; that Fajardo has occasionally used Sande's printed material in his own works~\cite{cruickshank1981}. Generally, as illustrated by Bedmar's case---who was briefly active in Seville, where the other aforementioned printers were established, before moving to Madrid---, italics types are considered by philologists to be much more discriminative than roman types, the latter having been in wider circulation throughout the Iberian peninsula; note how this is reflected in our experiments (see \cref{tab:metrics}). Hence our analysis on roman types can indicate loose geographical or temporal tendencies.

\paragraph{New attributions}

We present two novel philological findings emerging from our analysis. Initially suggested by the patterns observed in \Cref{fig:matrix1,fig:graph1}, they were subsequently validated using the established philological methodology.

The first is the case of the book numbered 51, which our results indicate sharing both roman and italic typefaces with known Fajardo books. The subsequent study by human experts revealed shared compositional patterns, letterhead typeface, and type damage, thereby providing unquestionable evidence that Fajardo's family can be extended to this \emph{suelta}.

Similarly, the particular position of books 80 and 81 in the graph (\cref{fig:graph1}) warranted a philological analysis. Beyond the roman typography, the study showed common page architecture, ornaments, and type damage, bringing proof of common printer origin; other clues pointed to clear separation from the other families mentioned in this paper, with a potential temporal gap of almost a century. Thus, these two \emph{sueltas} together create a new editorial family.

\subsection{Second corpus: reviewing attributions from italics} \label{ssec:discussion2}

\begin{figure}[htbp]
    \centering
    \subfloat[]{\label{subfig:matrix2}\includegraphics[height=4.7cm]{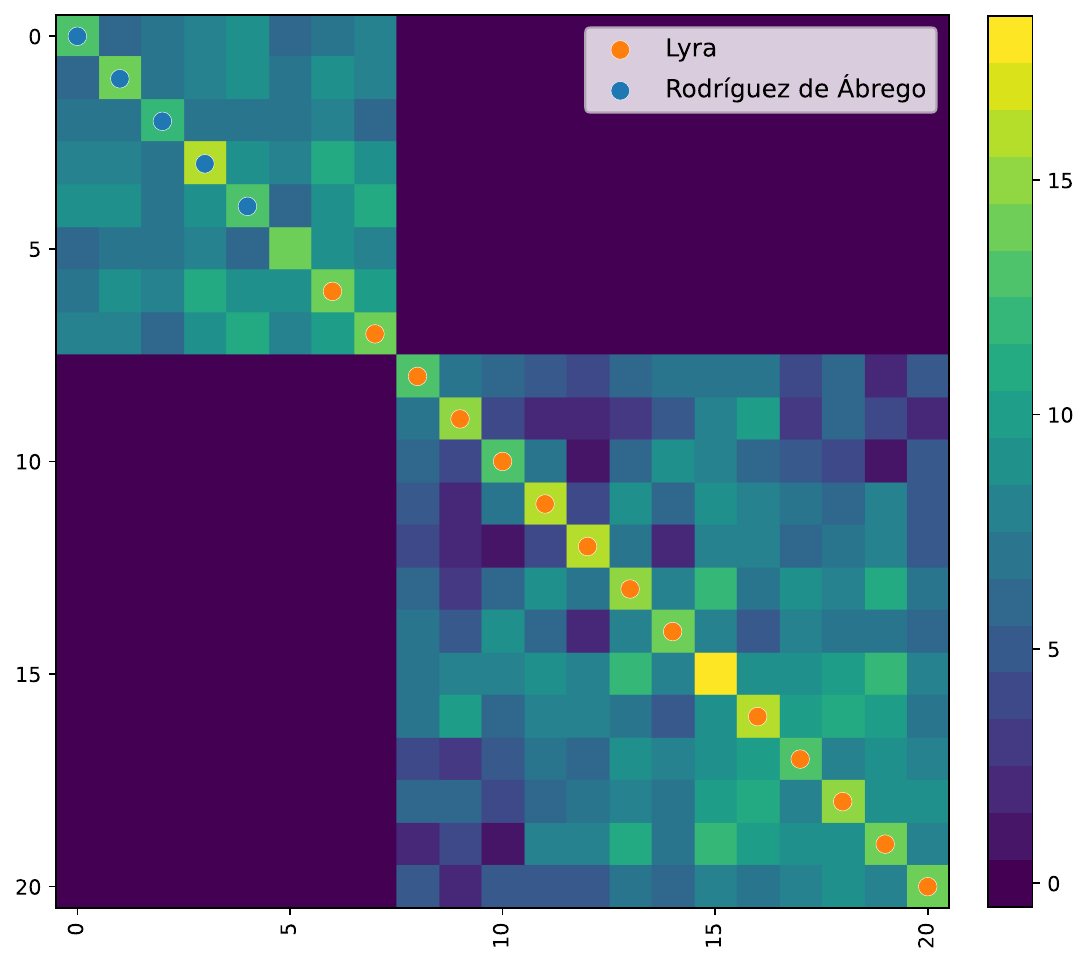}}
    \quad
    \subfloat[]{\label{subfig:graph2}\includegraphics[height=4.7cm]{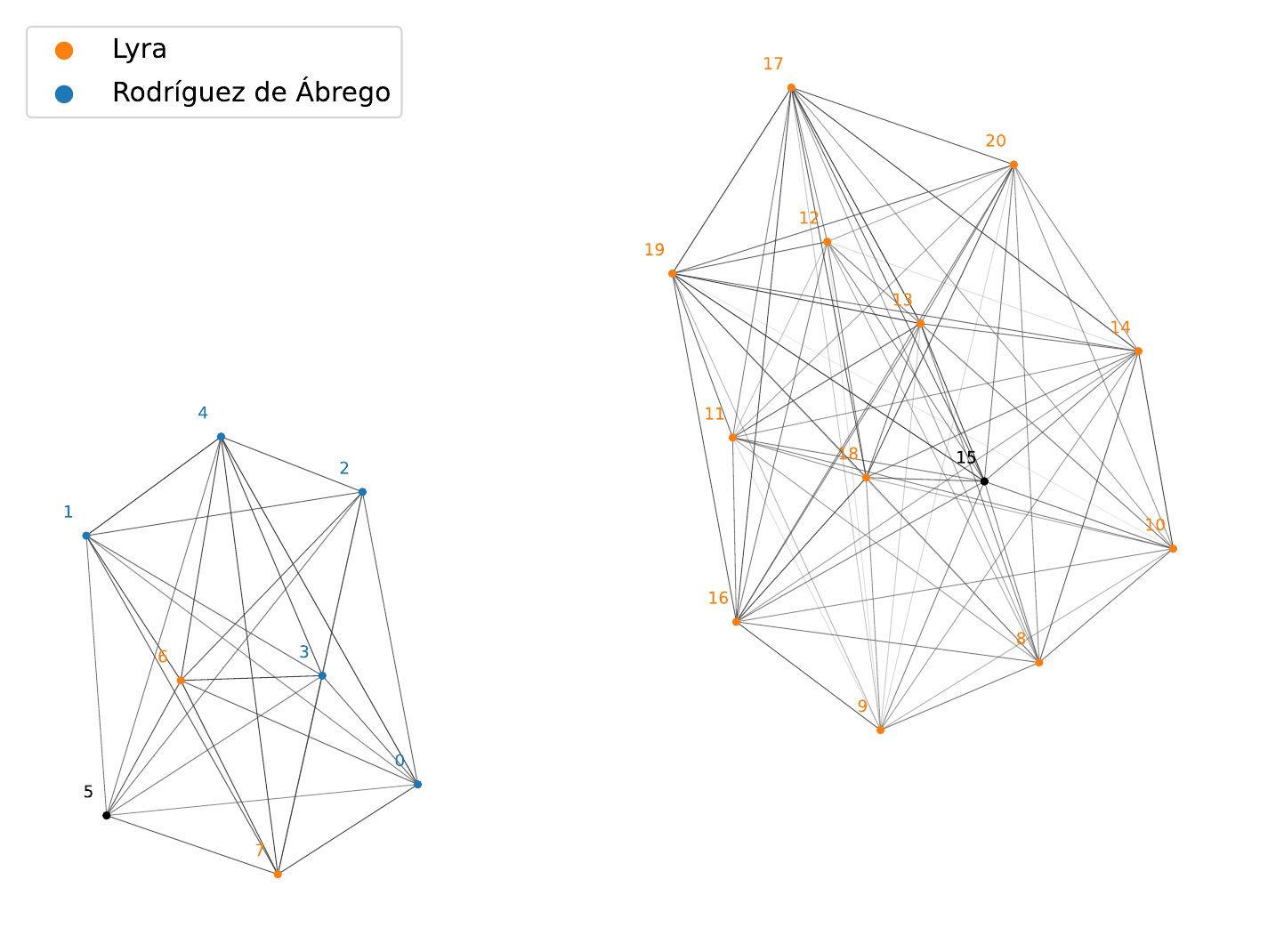}}
    \caption{Matrix representation of $\nhat_1$ (\protect\ref{subfig:matrix2}) and graph representation (\protect\ref{subfig:graph2}) for italic types on the corpus $\Bcal_2$ (\cref{ssec:corpus1}). Books are highlighted with their current printer attribution; as a consequence of these results, two attributions are to be revised, and two new attributions are suggested (see \cref{ssec:discussion2}).}
    \label{fig:matrixgraph2}
\end{figure}

\begin{figure}[t]
    \centering
    \includegraphics[width=\linewidth]{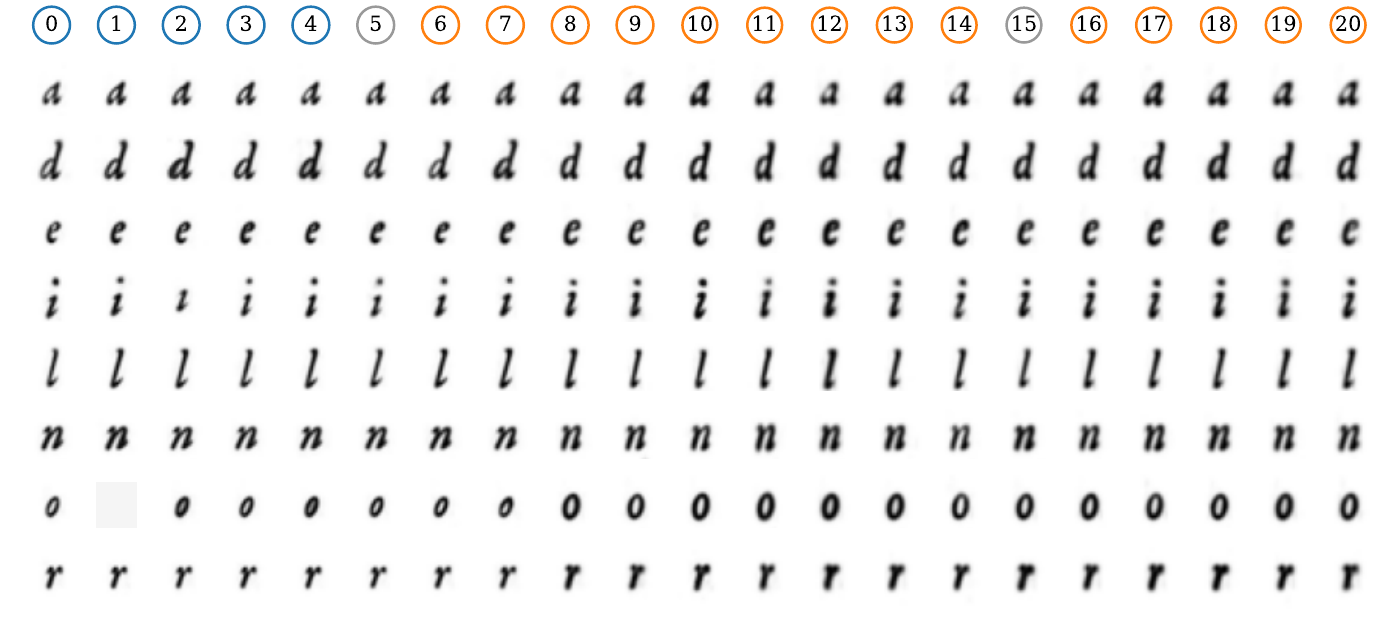}
    \caption{Illustrative summary of results from \cref{ssec:discussion2}: a set of representative character prototypes, for each book of $\Bcal_2$, as extracted by our algorithm. Current printer attribution is signaled as follows: Rodriguez de Ábrego (\textcolor{tabblue}{$\circ$}), Lyra (\textcolor{orange}{$\circ$}), unknown (\textcolor{gray}{$\circ$}). A clear distinction is seen between the Granjon italics (from $0$ until $7$) and the Haultin ones (from $8$ to the end). For example, Granjon is more calligraphic with a stronger slant, Haultin more regular and modular; the 'a' is more open in Granjon, the 'e' has a slanted loop vs. a more horizontal one in Haultin; Haultin is also slightly larger.}
    \label{fig:haultinVgranjon}
\end{figure}

As explained in \ref{ssec:corpus2}, the object of this experiment on the corpus $\Bcal_2$ is to study our method's applicability to a specific research problem motivated by the recent work of Casariego~\cite{casariegoinpress}: revisiting the current attribution of this group of chapbooks to Nicol{\'a}s Rodr{\'i}guez de {\'A}brego or Francisco de Lyra.

\paragraph{Corrections of current attributions}

The numerical results are shown in \Cref{fig:matrixgraph2}. In addition \Cref{fig:haultinVgranjon} displays samples of the character prototypes extracted by the result: it provides a concise preview of each book's ``typeface-wise family'', and at the same time yields an additional layer of interpretability and accountability to our method by allowing the user to check that the statistical analysis relies on correctly processed image data.

Analysis of these results reveals two large publishing families based on the use of italics: a first family to the top right of the graph with the cursive \textit{médiane italique grasse} (Haultin), which has been located in a publishing family of \textit{sueltas} printed by Lyra around 1635, and another to the bottom left with \textit{cicéro cursive} (Granjon), used by Rodríguez de Ábrego between 1647 and 1648~\cite[pp.997--998]{cruickshank2004}. We have verified these families using the traditional methodology.

Within these families, the algorithm detaches books 6 and 7, which were previously associated with the \textit{médiane italique grasse} (Haultin) family, and associates them with the \emph{cicéro cursive} (Granjon) group, a fact confirmed once again by analysis using the traditional method. Similarly, book 15, with no previously attributed publisher, appears to be related to this same type of italics. These typographical links make it necessary to disassociate books 6 and 7 from the Lyra printing press, and, thanks to the additional clue of a title-type breakage, it can be suggested that they came from the Rodríguez de Ábrego presses within the indicated time frame. Likewise, we confirm the linking of book 15 to the cursive \textit{médiane italique grasse} (Haultin).

\paragraph{Potential biases from image sourcing}

The image features of a document can depend heavily on the conditions under which it is preserved and digitized, which can vary from one archival institution to another. In this experiment, the editions in $\Bcal_2$ held by the Biblioteca Nacional de España (BNE) were scanned at a lower resolution than those from the Bibliothèque Nationale Universitaire of Strasbourg (BNU). Ruling out the possibility that such differences could bias our method's results is of the utmost importance if the method is to be deployed at scale across several archives.

It is therefore of particular interest to note that, in our method's results, no distinction appears between the three chapbooks newly attributed to Rodr{\'i}guez de {\'A}brego, which are from the BNE, and the five chapbooks previously associated with him, which are from the BNU, and they are all fully disconnected from the thirteen other \emph{sueltas}, which are from the BNE. This fact, together with our full corroboration of these results through a traditional philological analysis, gives us confidence that the algorithm does not group items based on the quality or methods of conservation and reproduction employed by a particular library, but rather distinguishes among different italic typefaces.

\section{Conclusion} \label{sec:conclusion}

We developed a method to estimate the proximity of typefaces between printed documents, based on image processing algorithms and statistical analysis. The proposed algorithmic tool can bring significant acceleration to the methodology for the analysis of ancient books: as shown in this paper with the study of two corpora of 17th century \emph{sueltas}, our method succeeds in assisting philologists in reconstructing relationships between printers, many of which are entirely unknown, by indicating potential associations to investigate, at a scale which, while still moderate, would already be overwhelming to the current manual and time-consuming processes.

Given that our method is built for scalability, the limiting factor appears to be the centralized collection of chapbook scans for processing. Looking ahead, we are building a larger corpus, on which our analysis should provide a more comprehensive picture of the Spanish 17th-century \textit{sueltas}. It will open up the prospect of establishing a general overview of the distribution of typefaces across Iberian printing, laying the groundwork for creating an atlas of Hispanic typography, a resource that does not currently exist and about which we know very little; and it will shed light on the dynamics of the publication of theatrical texts in early modern presses making possible the reconfiguration of the editorial and commercial map of books published in this century.


\begin{credits}
    \subsubsection{\ackname}
    The research that originated the results presented in this publication was partly supported by the Agencia Nacional de Investigación e Innovación of Uruguay\footnote{Under scholarship {\scriptsize POS\_EXT\_2023\_2\_180123}.}, by the France 2030 CollabNext project, and by the Spanish Ministry of Science (State Research Agency) through the ISTAE project\footnote{Loose prints from ancient Spanish theatre: integrated database of classical Spanish theatre. Second phase (ASODAT.\ Third phase) -- {\scriptsize KNOWLEDGE GENERATION 2022} -- Non-oriented research projects, reference {\scriptsize PID2022-136431NB-C66}. \url{https://istae.uv.es/}}. The work of J.-M.\ Morel is partially supported by RGC-GRF project 11309925, Mathematical Formalization of GIS. The experiments presented in this paper used ClusterUY \cite{clusterUy}. S. Mowlavi and D. Belzarena are researchers at the Centre Borelli, which is a member of ENS Paris-Saclay, Université Paris-Saclay, Université Paris Cité, CNRS, SSA, and INSERM. The authors are grateful to Aitor Artola, Natalia Bottaioli, Gabriele Facciolo, Marina Gardella, Roy He, Camilo Mari{\~n}o, Boshra Rajaei, and Ignacio Ram{\'i}rez for fruitful discussions.
    
    \subsubsection{\discintname}
    The authors have no competing interests to declare that are relevant to the content of this article.
\end{credits}
%
%
%
\bibliographystyle{splncs04}

\end{document}